\documentclass{article} 

\usepackage{algorithm}
\usepackage{algorithmic}
\usepackage{graphicx}
\usepackage{amssymb}
\usepackage{amsthm}
\usepackage{amsmath}
\usepackage{algorithm}
\usepackage{algorithmic}
\usepackage{url}

\newtheorem{thm}{Theorem}[section]
\newtheorem{lem}{Lemma}[section]

\newtheorem{assumption}{Assumption}[section]

\newcounter{append_lemma}
\newtheorem{appendLemma}[append_lemma]{Lemma}
\newcounter{append_cor}
\newtheorem{appendCorollary}[append_cor]{Corollary}

\newcommand{\lsp}[1]{\large\renewcommand{\baselinestretch}{#1}\normalsize}

\title{Sparse Trace Norm Regularization}

\author{\normalsize{Jianhui Chen\footnote{This work was done when the first author was a Ph.D. student at Arizona State University.}} \\ \small{GE Global Research, Niskayuna, NY}  \and \normalsize{Jieping Ye} \\ \small{Arizona State University, Tempe, AZ}}

\begin{document}

\maketitle

\begin{abstract}

We study the problem of estimating multiple predictive functions from
a dictionary of basis functions in the nonparametric regression
setting. Our estimation scheme assumes that each predictive function
can be estimated in the form of a linear combination of the basis
functions. By assuming that the coefficient matrix admits a sparse
low-rank structure, we formulate the function estimation problem as a
convex program regularized by the trace norm and the $\ell_1$-norm
simultaneously. We propose to solve the convex program using the
accelerated gradient (AG) method and the alternating direction method
of multipliers (ADMM) respectively; we also develop efficient
algorithms to solve the key components in both AG and ADMM. In
addition, we conduct theoretical analysis on the
proposed function estimation scheme: we derive a key property of the
optimal solution to the convex program; based on an assumption on the
basis functions, we establish a performance bound of the proposed
function estimation scheme (via the composite regularization).
Simulation studies demonstrate the effectiveness and efficiency of
the proposed algorithms.
\end{abstract}

\section{Introduction}

We study the problem of estimating multiple predictive functions from
noisy observations. Such a problem has received broad attention in
many areas of statistics and machine
learning~\cite{Bunea-lasso-colt06,Huang-sparse-icml09,Lounici-MTL-COLT08,Negahban-ICML10}.
This line of work can be roughly divided into two categories:
parametric estimation and non-parametric estimation; a common and
important theme for both categories is the appropriate assumption of
the structure in the model parameters (parametric setting) or the
coefficients of the dictionary (nonparametric setting).

There has been an enormous amount of literature on effective function
estimation based on different sparsity constraints, including the
estimation of the sparse linear regression via $\ell_1$-norm
penalty~\cite{bickel-simulLassoDantzig-AnnalsStat09,Bunea-lasso-colt06,Tibshirani-lasso-96,Zhang-sharp-AnnalsStat09},
and the estimation of the linear regression functions using group
lasso estimator~\cite{Huang-sparse-icml09,Lounici-MTL-COLT08}. More
recently, trace norm regularization has become a popular tool for
approximating a set of linear models and the associated low-rank
matrices in the high-dimensional
setting~\cite{Negahban-ICML10,Tsybakov-tracenorm-10}; the trace norm
is the tightest convex surrogate~\cite{Fazel-ACL-01} for the
(non-convex) rank function under certain conditions, encouraging the
sparsity in the singular values of the matrix of interest. One
limitation of the use of trace norm regularization is that the
resulting model is dense in general. However, in many real-world
applications~\cite{Pati-coherent-94}, the underlying structure of
multiple predictive functions may be sparse as well as low-rank; the
sparsity leads to explicitly interpretable prediction models and the
low-rank implies essential subspace structure information. Similarly,
the $\ell_1$-norm is the tightest convex surrogate for the non-convex
cardinality function~\cite{Boyd-Convex-04}, encouraging the sparsity
in the entries of the matrix. This motivates us to explore the use of
the combination of the trace norm and the $\ell_1$-norm as a
composite regularization (called sparse trace norm regularization) to
induce the desirable sparse low-rank structure.

Trace norm regularization (minimization) has been investigated
extensively in recent years. Efficient algorithms have been developed
for solving convex programs with trace norm
regularization~\cite{Toh-tracenorm-09,Fazel-ACL-01}; sufficient
conditions for exact recovery from trace norm minimization have been
established in~\cite{Benjamin-trace-siam07}; consistency of trace
norm minimization has been studied in~\cite{Bach-Tracenorm-JMLR08};
trace norm minimization has been applied for matrix
completion~\cite{Emmanuel-exact-08} and collaborative
filtering~\cite{Nathan-TraceNorm-nips04,Jason-tracenorm-icml05}.
Similarly, $\ell_1$-norm regularization has been well studied in the
literature, just to mention a few, from the efficient algorithms for
convex
optimization~\cite{Efron-LARS-04,Friedman-lasso-AnnalsStat07,Toh-tracenorm-09},
theoretical guarantee of the
performance~\cite{Candes-DecodingBYLP-05,Zhang-sharp-AnnalsStat09},
and model selection consistency~\cite{Zhao-lasso-06}.


In this paper, we focus on estimating multiple predictive functions
simultaneously from a finite dictionary of basis functions in the
nonparametric regression setting. Our function estimation scheme
assumes that each predictive function can be approximated using a
linear combination of those basis functions. By assuming that the
coefficient matrix of the basis functions admits a sparse low-rank
structure, we formulate the function estimation problem as a convex
formulation, in which the combination of the trace norm and the
$\ell_1$-norm is employed as a composite regularization to induce a
sparse low-rank structure in the coefficient matrix. The simultaneous
sparse and low-rank structure is different from the incoherent sparse
and low-rank structures studied
in~\cite{Candes:robustPCA:2011,Chandrasekaran-Incoherent-SYSID09}. We
propose to solve the function estimation problem using the
accelerated gradient method and the alternating direction method of
multipliers; we also develop efficient algorithms to solve the key
components involved in both methods. We conduct theoretical analysis
on the proposed convex formulation: we first present some basic
properties of the optimal solution to the convex formulation
(Lemma~\ref{lem:basic}); we then present an assumption associated
with the geometric nature of the basis functions over the prescribed
observations; based on such an assumption, we derive a performance
bound for the combined regularization for function estimation
(Theorem~\ref{thm:oracle}). We conduct simulations on benchmark data
to demonstrate the effectiveness and efficiency of the proposed
algorithms.
\vskip -0.05in
{\bf Notation} Denote $\mathbb{N}_n = \{1, \cdots, n\}$. For any
matrix $\Theta$, denote its trace norm by $\|\Theta\|_*$, i.e., the
sum of the singular values; denote its operator norm by
$\|\Theta\|_2$, i.e., the largest singular value; denote its
$\ell_1$-norm by $\|\Theta\|_1$, i.e., the sum of absolute value of
all entries.

\section{Problem Formulation}

Let $\{ (x_1, y_1), \cdots, (x_n, y_n) \} \subset \mathbb{R}^d \times
\mathbb{R}^k$ be a set of prescribed sample pairs (fixed design)
associated with $k$ unknown functions $\{f_1, \cdots, f_k \}$ as
\begin{equation} \label{eq:linear-reg-model}
y_{ij} = f_j (x_i) + w_{ij}, \quad i \in \mathbb{N}_n, \, j \in \mathbb{N}_k,
\end{equation}
where $f_j: \mathbb{R}^d \rightarrow \mathbb{R}$ is an unknown
regression function, $y_{ij}$ denotes the $j$-th entry of the
response vector $y_i \in \mathbb{R}^k$, and $w_{ij} \sim \mathcal{N}
(0, \sigma_w^2)$ is a stochastic noise variable. Let $X = [x_1,
\cdots, x_n]^T \in \mathbb{R}^{n \times d}$, $Y = [y_1, \cdots,
y_n]^T \in \mathbb{R}^{n \times k}$, and $W = \left( w_{ij}
\right)_{i,j} \in \mathbb{R}^{n \times k}$. Denoting
\begin{equation} \label{eq:def-F}
\mathcal{F} = \left( f_j (x_i) \right)_{i,j} \in \mathbb{R}^{n
\times k}, \quad i \in \mathbb{N}_n, \, j \in \mathbb{N}_k,
\end{equation}
we can rewrite Eq.~(\ref{eq:linear-reg-model}) in a compact form as
$Y = \mathcal{F} + W$.
Let $\{g_1, \cdots, g_h\}$ be a set of $h$ pre-specified basis
functions as $g_i: \mathbb{R}^{d} \rightarrow \mathbb{R}$, and let
$\Theta = [\theta_1, \cdots, \theta_k] \in \mathbb{R}^{h \times k}$
be the coefficient matrix. We define
\begin{equation} \label{eq:estimate-f}
\hat g_j (x) = \sum_{i=1}^h \theta_{ij} g_i (x), \quad j \in
\mathbb{N}_k,
\end{equation}
where $\theta_{ij}$ denotes the $i$-th entry in the vector
$\theta_j$. Note that in practice the basis functions $\{g_i\}$ can be
estimators from different methods, or different values of the tuning
parameters of the same method.

We consider the problem of estimating the unknown functions
$\{f_1, \cdots, f_k \}$ using the composite functions $\{\hat g_1,
\cdots, \hat g_k \}$ defined in Eq.~(\ref{eq:estimate-f}),
respectively. Denote
\begin{equation} \label{eq:def-G}
\mathcal{G}_X = \left( g_j (x_i) \right)_{i,j} \in \mathbb{R}^{n
\times h}, \quad i \in \mathbb{N}_n, \, j \in \mathbb{N}_h,
\end{equation}
and define the empirical error as
\begin{eqnarray} \label{eq:error}
\widehat S(\Theta) \hskip -0.03in = \hskip -0.03in \frac{1}{n k} \sum_{i=1}^{n} \sum_{j=1}^{k} \left(
\hat g_j(x_i) - y_{ij} \right)^2 \hskip -0.03in = \hskip -0.03in \frac{1}{N} \| \mathcal{G}_X
\Theta - Y\|_F^2, \hskip -0.09in
\end{eqnarray}
where $N = n \times k$. Our goal is to estimate the model parameter
$\Theta$ of a sparse low-rank structure from the given $n$ sample
pairs $\{ (x_i, y_i)\}_{i=1}^{n}$. Such a structure induces the
sparsity and the low rank simultaneously in a single matrix of
interest.

Given that the functions $\{f_1, \cdots, f_k \}$ are coupled via
$\Theta$ in some coherent sparse and low-rank structure, we propose to estimate $\Theta$ as
\begin{equation} \label{eq:def-minimizer-to-SDP}
\widehat \Theta =  \arg \min_\Theta \left( \widehat S(\Theta) +
\alpha \|\Theta\|_* +
\beta \|\Theta\|_1 \right),
\end{equation}
where $\alpha$ and $\beta$ are regularization parameters (estimated via cross-validation), and the linear combination of
$\| \Theta\|_*$ and $\| \Theta\|_1$ is used to induce the sparse low-rank structure in $\Theta$. The optimization problem in
Eq.~(\ref{eq:def-minimizer-to-SDP}) is non-smooth convex and hence admits a
globally optimal solution; it can be solved using many
sophisticated optimization
techniques~\cite{Toh-SDPT3-OptSW99,Fazel-ACL-01};
in Section~\ref{sec:algorithm}, we propose to apply the accelerated gradient method~\cite{Nesterov-IntrodConvexOpt-note98} and the alternating direction method of multipliers~\cite{Boyd-ADMM-10} to solve the optimization problem in
Eq.~(\ref{eq:def-minimizer-to-SDP}).

\section{Optimization Algorithms} \label{sec:algorithm}

In this section, we consider to apply the accelerated gradient~(AG)
algorithm~\cite{Beck-fast-09,Nesterov-IntrodConvexOpt-note98,Nesterov:2007}
and the alternating direction method of multipliers
(ADMM)~\cite{Boyd-ADMM-10}, respectively, to solve the (non-smooth
and convex) optimization problem in
Eq.~(\ref{eq:def-minimizer-to-SDP}). We also develop efficient
algorithms to solve the key components involved in both AG and ADMM.

\subsection{Accelerated Gradient Algorithm}

The AG algorithm has attracted extensive attention in the machine learning community due to its optimal convergence rate among all first order techniques and its ability of dealing with large scale data. The general scheme in AG for solving Eq.~(\ref{eq:def-minimizer-to-SDP}) can be described as below: at the $k$-th iteration, the intermediate (feasible) solution $\Theta_k$ can be obtained via
\begin{equation} \label{eq:original-form-key-step}
\Theta_k = \arg \min_\Theta \left( \frac{\gamma_k}{2} \left \| \Theta - \hskip -0.05in \left( \Phi_k - \frac{1}{\gamma_k} \nabla \widehat S(\Phi_k) \right) \hskip -0.02in \right \|_F^2 \hskip -0.1in + \alpha \| \Theta \|_* + \beta \| \Theta \|_1 \right),
\end{equation}
where $\Phi_k$ denotes a searching point constructed on the intermediate solutions from previous iterations, $\nabla \widehat S(\Phi_k)$ denotes the derivative of the loss function in Eq.~(\ref{eq:error}) at $\Phi_k$, and $\gamma_k$ specifies the step size which can be determined by iterative increment until the condition
\begin{equation*}
\widehat S(\Theta_k) \le \widehat S(\Phi_k) + \langle \nabla f(\Phi_k), \Theta_k - \Phi_k \rangle + \frac{\gamma_k}{2} \| \Theta_k - \Phi_k \|_F^2
\end{equation*}
is satisfied. The operation in Eq.~(\ref{eq:original-form-key-step}) is commonly referred to as proximal operator~\cite{Moreau:1965}, and its efficient computation is critical for the practical convergence of the AG-type algorithm. Next we present an efficient alternating optimization procedure to solve Eq.~(\ref{eq:original-form-key-step}) with a given $\gamma_k$.


\subsubsection{Dual Formulation}

The problem in Eq.~(\ref{eq:original-form-key-step}) is not easy to solve directly; next we show that this problem can be efficiently solved in its dual form. By reformulating $\| \Theta \|_*$ and $\|\Theta|_1$ into the equivalent dual forms, we convert Eq.~(\ref{eq:original-form-key-step}) into a max-min formulation as
\begin{eqnarray} \label{Eq:max-min}
\max_{L, S} \min_\Theta  \,\, \| \Theta - \widehat \Phi \|_F^2 + \widehat \alpha \langle  L, \Theta \rangle + \widehat \beta \langle S, \Theta \rangle, \,\,\,\, \mbox{subject to}  \,\, \|L\|_2 \le 1, \,\, \|S\|_\infty \le 1,
\end{eqnarray}
where $\widehat \Phi = \Phi_k - \nabla \widehat S(\Phi_k) / \gamma_k$, $\widehat \alpha = 2 \alpha /\gamma_k$, and $\widehat \beta = 2 \beta /\gamma_k$. It can be verified that in Eq.~(\ref{Eq:max-min}) the Slater condition is satisfied and strong duality holds~\cite{Boyd-Convex-04}. Also the optimal $\Theta$ can be expressed as a function of $L$ and $S$ given by
\begin{equation} \label{eq:opt-theta-max-min}
\Theta =  \widehat \Phi - \frac{1}{2} (\widehat \alpha L +  \widehat \beta S).
\end{equation}
By substituting Eq.~(\ref{eq:opt-theta-max-min}) into Eq.~(\ref{Eq:max-min}),
we obtain the dual form of Eq.~(\ref{eq:original-form-key-step}) as
\begin{eqnarray} \label{eq:two-block}
\min_{L, S}     \,\, \| \widehat \alpha L + \widehat \beta S - 2 \widehat \Phi \|_F^2, \,\,\,\, \mbox{subject to}   \,\, \|L\|_2 \le 1, \,\, \|S\|_\infty \le 1.
\end{eqnarray}

\subsubsection{Alternating Optimization} \label{subsec:alt-opt}

The optimization problem in Eq.~(\ref{eq:two-block}) is smooth convex and it has two optimization variables. For such type of problems, coordinate descent (CD) method is routinely used to compute its globally optimal solution~\cite{Grippoa-cd-Orletter00}.
To solve Eq.~(\ref{eq:two-block}), the CD method alternatively optimizes one of the two variables with the other variable fixed. Our analysis below shows that the variables $L$ and $S$ in Eq.~(\ref{eq:two-block}) can be optimized efficiently.
Note that the convergence rate of the CD method is not known, however, it converges very fast in practice (less than $10$ iterations in our experiments).

{\bf Optimization of L} For a given $S$, the variable $L$ can be optimized via solving the following problem:
\begin{eqnarray} \label{eq:opt-low-rank}
\min_L                  \,\, \| L - \widehat L \|_F^2, \,\,\,\, \mbox{subject to}   \,\, \|L\|_2 \le 1,
\end{eqnarray}
where $\widehat L = ( 2 \widehat \Phi - \widehat \beta S) / \widehat \alpha$. The optimization on $L$ above can be interpreted as computing an optimal projection of a given matrix over a unit spectral norm ball. Our analysis shows that the optimal solution to Eq.~(\ref{eq:opt-low-rank}) can be expressed in an analytic form as summarized in the following theorem.
\begin{thm} \label{thm:projection-spectral-norm}
For arbitrary $\widehat L \in \mathbb{R}^{h \times k}$ in Eq.~(\ref{eq:opt-low-rank}), denote its SVD by $\widehat L = U \Sigma V^T$, where $r = \mbox{rank} (\widehat L)$, $U \in \mathbb{R}^{h \times r}$, $V \in \mathbb{R}^{k \times r}$, and $\Sigma = \mbox{diag} \left(\sigma_1, \cdots, \sigma_r\right) \in \mathbb{R}^{r \times r}$. Let
$\hat \sigma_i^* = \min \left( \sigma_i, 1  \right), \,\,\, i = 1, \cdots, r$. Then the optimal solution to Eq.~(\ref{eq:opt-low-rank}) is given by
\begin{equation} \label{eq:solution-L}
L^* = U \hat \Sigma V^T, \,\, \hat \Sigma = \mbox{diag} \left(\hat \sigma_1^*, \cdots, \hat \sigma_r^* \right).
\end{equation}
\end{thm}
\begin{proof}
Assume the existence of a set of left and right singular vector pairs shared by the optimal $L^*$ to Eq.~(\ref{eq:opt-low-rank}) and the given $\widehat L$ for their non-zero singular values. Under such an assumption, it can be verified that the singular values of $L^*$ can be obtained via
\begin{eqnarray*}
\min_{\{\hat \sigma_i\}} \,\, \left( \hat \sigma_i - \sigma_i \right)^2, \,\,\,\, \mbox{subject to}   \,\,  0 \le \hat \sigma_i \le 1, \,\, i = 1, \cdots, r,
\end{eqnarray*}
to which the optimal solution is given by $\hat \sigma_i^* = \min (\sigma_i, 1)~(\forall i)$; hence the expression of $L^*$ coincides with Eq.~(\ref{eq:solution-L}). Therefore, all that remains is to show that our assumption (on the left and right singular vector pairs of $L^*$ and $\widehat L$) holds.

Denote the Lagrangian associated with the problem in Eq.~(\ref{eq:opt-low-rank}) as $h(L, \lambda) = \|L - \widehat L\|_F^2 + \lambda \left( \|L \|_2 - 1 \right)$,
where $\lambda$ denotes the dual variable. Since $\bf 0$ is strictly feasible in Eq.~(\ref{eq:opt-low-rank}), namely, $\|{\bf 0}\|_2 < 1$, strong duality holds for Eq.~(\ref{eq:opt-low-rank}). Let $\lambda^*$ be the optimal dual variable to Eq.~(\ref{eq:opt-low-rank}). Therefore we have $L^* = \arg \min_{L} h(L, \lambda^*)$.
It is well known that $L^*$ minimizes $h(L, \lambda^*)$ if and only if $\bf 0$ is a subgradient of $h(L, \lambda^*)$ at $L^*$, i.e.,
\begin{eqnarray} \label{eq:derive-subgradient}
{\bf 0} \in 2 ( L^* - \widehat L ) + \lambda^* \partial \|L^* \|_2.
\end{eqnarray}
For any matrix $Z$, the subdifferential of $\|Z\|_2$ is given by~\cite{Watson-subdiff-92}
$\partial \|Z\|_2  =  \mbox{conv} \left\{ u_z v_z^T : \|u_z\|  =  \|v_z\|  =  1, Z v_z  =  \|Z\|_2 u_z \right\}$,
where $\mbox{conv} \{c\}$ denotes the convex hull of the set $c$. Specifically, any element of $\partial \|Z\|_2$ has the form
\begin{equation*}
\sum_{i} \alpha_i u_{zi} v_{zi}^T, \,\, \alpha_i \ge 0, \,\, \sum_{i} \alpha_i = 1,
\end{equation*}
where $u_{zi}$ and $v_{zi}$ are any left and right singular vectors of $Z$ corresponding to its largest singular value (the top singular values may share a common value). From Eq.~(\ref{eq:derive-subgradient}) and the definition of $\partial \|Z\|_2$, there exist $\{\hat \alpha_i\}$ such that
$\hat \alpha_i > 0, \,\, \sum_{i} \hat \alpha_i = 1, \,\, \sum_{i} \hat \alpha_i u_{li} v_{li}^T \in \partial \| L^* \|_2$, and
\begin{equation} \label{eq:svd-L-star-L-hat}
\widehat L = L^* + \frac{\lambda^*}{2} \sum_{i} \hat \alpha_i u_{li} v_{li}^T,
\end{equation}
where $u_{li}$ and $v_{li}^T$ correspond to any left and right singular vectors of $L^*$ corresponding to its largest singular value. Since $\lambda^*, \hat \alpha_i > 0$, Eq.~(\ref{eq:svd-L-star-L-hat}) verifies the existence of a set of left and right singular vector pairs shared by $L^*$ and $\widehat L$. This completes the proof.
\end{proof}

{\bf Optimization of S} For a given $L$, the variable $S$ can be optimized via solving the following problem:
\begin{eqnarray} \label{eq:opt-sparse}
\min_S     \,\, \| S - \widehat S \|_F^2, \,\,\,\, \mbox{subject to}  \,\, \|S\|_\infty \le 1,
\end{eqnarray}
where $\widehat S = ( 2 \widehat \Phi - \widehat \alpha L ) / \widehat \beta$. Similarly, the optimization on $S$ can be interpreted as computing a projection of a given matrix over an infinity norm ball. It also admits an analytic solution as summarized in the following theorem.
\begin{lem} \label{lem:projection-infinity-norm}
For any matrix $\widehat S$, the optimal solution to Eq.~(\ref{eq:opt-sparse}) is given by
\begin{equation}
S^* = \mbox{sgn} (\widehat S) \circ \min (|\widehat S|, {\bf 1}),
\end{equation}
where $\circ$ denotes the component-wise multiplication operator, and $\bf 1$ denotes the matrix with entries $1$ of appropriate size.
\end{lem}

\subsection{Alternating Direction Method of Multipliers}

The ADMM algorithm~\cite{Boyd-ADMM-10} is suitable for dealing with non-smooth (convex) optimizations problems, as it blends the decomposability of dual ascent with the superior convergence of the method of multipliers. We present two implementations of the ADMM algorithm for solving Eq.~(\ref{eq:def-minimizer-to-SDP}).
Due to the space constraint, we move the detailed discussion of two ADMM implementations to the supplemental material.

\section{Theoretical Analysis}

In this section, we present a performance bound for the function
estimation scheme in Eq.~(\ref{eq:estimate-f}). Such a performance bound measures how well the estimation scheme can approximate the regression functions $\{f_j\}$ in Eq.~(\ref{eq:def-F}) via the sparse low-rank coefficient $\Theta$.

\subsection{Basic Properties of the Optimal Solution}

We first present some basic properties of the optimal
solution defined in Eq.~(\ref{eq:def-minimizer-to-SDP}); these
properties are important building blocks of our following theoretical analysis.
\begin{lem} \label{lem:basic}
Consider the optimization problem in
Eq.~(\ref{eq:def-minimizer-to-SDP}) for $h, k \ge 2$ and $n \ge 1$.
Given $n$ sample pairs as $X = [x_1, \cdots, x_n]^T \in
\mathbb{R}^{n \times d}$ and $Y = [y_1, \cdots, y_n]^T \in
\mathbb{R}^{n \times k}$. Let $\mathcal{F}$ and $\mathcal{G}_X$ be
defined in Eq.~(\ref{eq:def-F}) and Eq.~(\ref{eq:def-G}),
respectively; let $\sigma_{X(l)}$ be the largest
singular values of $\mathcal{G}_X$.
Assume that $W \in \mathbb{R}^{n \times k}$ has independent and
identically distributed (i.i.d.) entries as $w_{ij} \sim \mathcal{N}
(0, \sigma_w^2 )$. Take
\begin{equation} \label{eq:def-lambda}
\alpha + \beta = \frac{2 \sigma_{X(l)} \sigma_w \sqrt{n}}{N} \left( 1 +
\sqrt{\frac{k}{n}} + t \right),
\end{equation}
where $N = n \times k$ and $t$ is a universal constant. Then with
probability of at least $1 - \exp \left( - n t^2 / 2 \right)$, for
the minimizer $\widehat \Theta$ in Eq.~(\ref{eq:def-minimizer-to-SDP})
and any $\Theta \in \mathbb{R}^{h \times k}$, we have
\begin{eqnarray} \label{eq:lem-1}
\frac{1}{N} \| \mathcal{G}_X \widehat \Theta - \mathcal{F} \|_F^2  \le  \frac{1}{N} \| \mathcal{G}_X \Theta  - \mathcal{F} \|_F^2 + 2 \alpha \|\mathcal{S}_0(\widehat \Theta - \Theta)\|_* + 2 \beta \|(\widehat \Theta - \Theta)_{J(\Theta)}\|_1,
\end{eqnarray}
where $\mathcal{S}_0$ is an operator defined in Lemma~\ref{lem:sep-singular-value} of the supplemental material.
\end{lem}
\begin{proof}
From the definition
of $\widehat \Theta$ in Eq.~(\ref{eq:def-minimizer-to-SDP}), we have $\widehat S(\widehat \Theta) + \alpha \|\widehat \Theta\|_* + \beta \| \widehat \Theta \|_1 \le \widehat S(\Theta) + \alpha \|\Theta\|_* + \beta \| \Theta \|_1$.
By substituting $Y = \mathcal{F} + W$ and
Eq.~(\ref{eq:error}) into the previous inequality, we have
\begin{eqnarray*}
\frac{1}{N} \| \mathcal{G}_X \widehat \Theta \hskip-0.02in - \hskip-0.02in \mathcal{F} \|_F^2 \hskip-0.02in \le \hskip-0.02in \frac{1}{N} \| \mathcal{G}_X \Theta \hskip-0.02in - \hskip-0.02in \mathcal{F} \|_F^2 \hskip-0.02in + \hskip-0.02in \frac{2}{N}\langle W,  \mathcal{G}_X ( \widehat \Theta \hskip-0.02in - \hskip-0.02in \Theta )
\rangle \hskip-0.02in + \hskip-0.02in \alpha \left( \hskip-0.02in \|\Theta\|_* \hskip-0.02in - \hskip-0.02in \|\widehat \Theta\|_* \hskip-0.02in \right) \hskip-0.02in + \hskip-0.02in \beta \left( \hskip-0.02in \|\Theta\|_1 \hskip-0.02in - \hskip-0.02in \|\widehat \Theta\|_1 \hskip-0.02in \right).
\end{eqnarray*}
Define the random event
\begin{equation} \label{eq:event-A}
\mathcal{A} = \left\{  \frac{1}{N} \| \mathcal{G}_X^T W \|_2  \le
\frac{\alpha + \beta}{2} \right\}.
\end{equation}
Taking $\alpha + \beta$ as the value in Eq.~(\ref{eq:def-lambda}), it
follows from Lemma~\ref{lem:event-A-hold} of the supplemental materia that $\mathcal{A}$ holds with probability of at least $1 - \exp \left(
- n t^2 / 2 \right)$. Therefore, we have
\begin{eqnarray*}
\hskip -0.1in &      & \hskip -0.1in \langle W,  \mathcal{G}_X ( \widehat \Theta - \Theta ) \rangle = \frac{\alpha + \beta }{\alpha + \beta} \langle W,  \mathcal{G}_X ( \widehat \Theta - \Theta ) \rangle \\
\hskip -0.1in &  \le & \hskip -0.1in \frac{\alpha}{\alpha + \beta} \| \mathcal{G}_X^T W \|_2  \| \widehat \Theta - \Theta \|_*  +  \frac{\beta}{\alpha + \beta} \| \mathcal{G}_X^T W \|_\infty  \| \widehat \Theta - \Theta \|_1 \le  \frac{N}{2} \left( \alpha \| \widehat \Theta - \Theta \|_* + \beta \| \widehat \Theta - \Theta \|_1 \right),
\end{eqnarray*}
where the second inequality follows from $\| \mathcal{G}_X^T W \|_2 \ge \| \mathcal{G}_X^T W \|_\infty$.
Therefore, under $\mathcal{A}$, we have
\begin{eqnarray*} \label{eq:lem-1-0}
\hskip -0.1in &     & \hskip -0.1in \frac{1}{N} \| \mathcal{G}_X  \widehat \Theta - \mathcal{F} \|_F^2 \\
\hskip -0.1in & \le & \hskip -0.1in \frac{1}{N} \| \mathcal{G}_X \Theta - \mathcal{F} \|_F^2 + \alpha \| \widehat \Theta - \Theta \|_* + \beta \| \widehat \Theta - \Theta \|_1 + \alpha \left( \|\Theta\|_* - \|\widehat \Theta\|_* \right) + \beta \left( \|\Theta\|_1 - \|\widehat \Theta\|_1 \right).
\end{eqnarray*}
From Corollary~\ref{cor:sep-tracenorm} and Lemma~\ref{lem:sep-nonzero-entry} of the supplemental material, we complete the proof.
\end{proof}

\subsection{Main Assumption}

We introduce a key assumption on the dictionary of basis functions $\mathcal{G}_X$. Based on such an assumption, we derive a performance bound for the sparse trace norm regularization formulation in Eq.~(\ref{eq:def-minimizer-to-SDP}).
\begin{assumption} \label{assump}
For a matrix pair $\Theta$ and
$\Delta$ of size $h \times k$, let $s \le \min(h, k)$ and $q \le h \times k$. We assume that there exist constants $\kappa_1 (s)$ and $\kappa_2 (q)$
such that
\begin{eqnarray} \label{eq:assump-1}
\kappa_1 (s) \triangleq \min_{\Delta \in \mathcal{R} (s, q)}
\frac{\|\mathcal{G}_X \Delta \|_F}{\sqrt{N} \|\mathcal{S}_0 (\Delta)
\|_*}
> 0,  \,\,
\kappa_2 (q)  \triangleq  \min_{\Delta \in \mathcal{R} (s, q)} \frac{ \| \mathcal{G}_X \Delta \|_F}{ \sqrt{N} \| \Delta_{J(\Theta)}\|_1 } > 0,
\end{eqnarray}
where the restricted set $\mathcal{R} (s, q)$ is defined as
\begin{eqnarray*}
\mathcal{R} (s, q) = \left\{ \Delta \in \mathbb{R}^{h \times k}, \Theta \in \mathbb{R}^{h \times k} \,
| \, \Delta \ne 0, \,\, \mbox{rank} (\mathcal{S}_0 (\Delta)) \le s,
\,\, \left|J(\Theta)\right| \le q \right\},
\end{eqnarray*}
and $| J(\Theta)|$ denotes the number of nonzero entries in the matrix $\Theta$.
\end{assumption}
Our assumption on $\kappa_1 (s)$ in Eq.~(\ref{eq:assump-1}) is closely related to but less restrictive than the RSC condition used in~\cite{Negahban-ICML10}; its denominator is only a part of the one in RSC and in a different matrix norm as well. Our assumption on $\kappa_2 (q)$ is similar to the RE condition used in~\cite{bickel-simulLassoDantzig-AnnalsStat09} except that its denominator is in a different matrix norm; our assumption can also be implied by sufficient conditions similar to the ones in~\cite{bickel-simulLassoDantzig-AnnalsStat09}.

\subsection{Performance Bound}

We derive a performance bound for the sparse trace norm structure obtained by solving Eq.~(\ref{eq:def-minimizer-to-SDP}). This bound measures how well the optimal $\widehat \Theta$ can be used to approximate $\mathcal{F}$  by evaluating the averaged estimation error, i.e., $\|\mathcal{G}_X \widehat \Theta - \mathcal{F}\|_F^2/N$.
\begin{thm} \label{thm:oracle}
Consider the optimization problem in
Eq.~(\ref{eq:def-minimizer-to-SDP}) for $h, k \ge 2$ and $n \ge 1$.
Given n sample pairs as $X = [x_1, \cdots, x_n]^T \in \mathbb{R}^{n
\times d}$ and $Y = [y_1, \cdots, y_n]^T \in \mathbb{R}^{n \times
k}$, let $\mathcal{F}$ and $\mathcal{G}_X$ be defined in
Eqs.~(\ref{eq:def-F}) and~(\ref{eq:def-G}), respectively; let
$\sigma_{X(l)}$ be the largest singular value of $\mathcal{G}_X$. Assume that $W \in
\mathbb{R}^{n \times k}$ has i.i.d. entries as $w_{ij} \sim
\mathcal{N} (0, \sigma_w^2 )$. Take $\alpha + \beta$ as the value in Eq.~(\ref{eq:def-lambda}).
Then with
probability of at least $1 - \exp \left( - n t^2 / 2
\right)$, for the minimizer $\widehat \Theta$ in
Eq.~(\ref{eq:def-minimizer-to-SDP}), we have
\begin{eqnarray}
\label{eq:oracle-ineq-result}
\frac{1}{N} \| \mathcal{G}_X \widehat \Theta - \mathcal{F} \|_F^2 \le  (1+\epsilon)  \inf_{\Theta} \left \{\frac{1}{N} \| \mathcal{G}_X \Theta -
\mathcal{F} \|_F^2 \right\} + \mathcal{E} (\epsilon) \left( \frac{\alpha^2}{\kappa_1^2 (2 r)} + \frac{\beta^2}{\kappa_2^2 (c)} \right),
\end{eqnarray}
where $\inf$ is taken over all $\Theta \in \mathbb{R}^{h \times k}$
with $\mbox{rank} (\Theta) \le r$ and $| J(\Theta) | \le c$, and $\mathcal{E} (\epsilon)
> 0$ is a constant depending only on $\epsilon$.
\end{thm}
\begin{proof}
Denote $\Delta = \widehat \Theta - \Theta$ in Eq.~(\ref{eq:lem-1}). We have
\begin{eqnarray} \label{eq:perform-step-1}
\hskip -0.1in \frac{1}{N} \| \mathcal{G}_X \widehat \Theta - \mathcal{F} \|_F^2  \le  \frac{1}{N} \| \mathcal{G}_X \Theta - \mathcal{F} \|_F^2 + 2 \alpha \| \mathcal{S}_0 (\Delta) \|_* + 2 \beta \| \Delta_{J(\Theta)}\|_1.
\end{eqnarray}
Given $\mathcal{S}_0 (\Delta) \le 2 r$ (from Lemma~\ref{lem:sep-singular-value} of the supplemental material) and $| J(\Theta) | \le c$, we derive upper bounds on the components $2 \alpha \| \mathcal{S}_0 (\Delta) \|_*$ and $2 \beta \| \Delta_{J(\Theta)}\|_1$ over the restrict set $\mathcal{R} (2 r, c)$ based on Assumptions~\ref{assump}, respectively. It follows that
\begin{eqnarray} \label{eq:bound-1}
2 \alpha \|\mathcal{S}_0 (\Delta) \|_*  & \le &  \frac{2 \alpha}{\kappa_1 (2 r) \sqrt{N}} \| \mathcal{G}_X (\widehat \Theta - \Theta) \|_F \le \frac{2 \alpha}{\kappa_1(2 r) \sqrt{N}} \left( \| \mathcal{G}_X \widehat \Theta - \mathcal{F} \|_F + \|
\mathcal{G}_X \Theta - \mathcal{F} \|_F  \right) \nonumber \\
& \le &   \frac{ \alpha^2 {\tau}}{ \kappa_1^2 (2 r)} + {{\frac{1}{{N \tau}}}} \| \mathcal{G}_X \widehat \Theta - \mathcal{F} \|_F^2 + \frac{ \alpha^2 {\tau}}{\kappa_1^2 (2 r)}
 + {{\frac{1}{{ N \tau}}}} \| \mathcal{G}_X \Theta
- \mathcal{F} \|_F^2,
\end{eqnarray}
where the last inequality above follows from $2 a b \le a^2 {\tau} + b^2 / {\tau}$ for $\tau > 0$.
Similarly, we have
\begin{eqnarray} \label{eq:bound-2}
2 \beta \| \Delta_{J(\Theta)} \|_1 & \le &   \frac{\beta^2 \tau}{ \kappa_2^2 (c)}
+ {{\frac{1}{N \tau}}} \| \mathcal{G}_X \widehat \Theta - \mathcal{F} \|_F^2 + \frac{\beta^2 \tau}{ \kappa_2^2 (c)}
+ {{\frac{1}{{N  \tau}}}} \| \mathcal{G}_X \Theta
- \mathcal{F} \|_F^2. 
\end{eqnarray}
Substituting Eqs.~(\ref{eq:bound-1}) and~(\ref{eq:bound-2}) into Eq.~(\ref{eq:perform-step-1}), we have
\begin{eqnarray*}
\frac{1}{N} \| \mathcal{G}_X \widehat \Theta - \mathcal{F} \|_F^2 & \le & \frac{\tau + 2}{(\tau - 2)N} \| \mathcal{G}_X \Theta - \mathcal{F} \|_F^2 + \frac{2 \tau^2}{\tau -2} \left( \frac{\alpha^2}{\kappa_1^2 (2 r)} + \frac{\beta^2}{\kappa_2^2 (c)} \right).
\end{eqnarray*}
Setting $\tau = 2 + 4 / \epsilon$ and $\mathcal{E} (\epsilon) = 2 (\epsilon + 2)^2 / \epsilon$ in the inequality above, we complete the proof.
\end{proof}
By choosing specific values for $\alpha$ and $\beta$, we can refine the performance bound described in Eq.~(\ref{eq:oracle-ineq-result}). 
It follows from Eq.~(\ref{eq:def-lambda}) we have
\begin{equation}
\min_{\alpha, \beta, \alpha + \beta = \gamma} \left( \frac{\alpha^2}{\kappa_1^2 (2 r)} + \frac{\beta^2}{\kappa_2^2 (c)} \right) =  \frac{\gamma^2}{\kappa^2_1 (2 r) + \kappa^2_2 (c)}, \,\, \gamma = \frac{2 \sigma_{X(l)} \sigma_w \sqrt{n}}{N} \left( 1 +
\sqrt{\frac{k}{n}} + t \right),
\end{equation}
where the equality of the first equation is achieved by setting $\alpha$ and $\beta$ proportional to $\kappa^2_1 (2 r)$ and $\kappa^2_2 (q)$, i.e., $\alpha = {\gamma \kappa^2_1(2 r)}/\left(\kappa^2_1(2 r) + \kappa^2_2(c)\right)$ and $\beta = \gamma \kappa^2_2(c) / \left( \kappa^2_1(2 r) + \kappa^2_2(c) \right)$. Thus the performance bound in Eq.~(\ref{eq:oracle-ineq-result}) can be refined as
\begin{eqnarray*}
\frac{1}{N} \| \mathcal{G}_X \widehat \Theta - \mathcal{F} \|_F^2 \le (1+\epsilon)  \inf_{\Theta} \left \{\frac{1}{N} \| \mathcal{G}_X \Theta -
\mathcal{F} \|_F^2 \right\} + \frac{4 \mathcal{E} (\epsilon) \sigma^2_{X(l)} \sigma_w^2 n}{N^2 \left(\kappa^2_1 (2 r) + \kappa^2_2 (c)\right) } \left( 1 +
\sqrt{\frac{k}{n}} + t \right)^2.
\end{eqnarray*}
Note that the performance bound above is independent of the value of $\alpha$ and $\beta$, and it is tighter than the one described in Eq.~(\ref{eq:oracle-ineq-result}).

\section{Experiments}

In this section, we evaluate the effectiveness of the sparse trace norm regularization formulation in Eq.~(\ref{eq:def-minimizer-to-SDP}) on benchmark data sets; we also conduct numerical studies on the convergence of AG and two ADMM implementations including ADMM$1$ and ADMM$2$ (see details in Section~E of the supplemental material) for solving Eq.~(\ref{eq:def-minimizer-to-SDP}) and the convergence of the alternating optimization algorithm for solve Eq.~(\ref{eq:two-block}). Note that we use the least square loss for the following experiments.

\noindent {\bf Performance Evaluation} We apply the sparse trace norm regularization formulation (S.TraceNorm) on multi-label classification problems, in comparison with the trace norm regularization formulation (TraceNorm) and the $\ell_1$-norm regularization formulation (OneNorm). AUC, Macro F$1$, and Micro F$1$ are used as the classification performance measures. Four benchmark data sets, including Business, Arts, and Health from Yahoo webpage data sets~\cite{Ueda-YahooData-KDD02} and Scene from LIBSVM multi-label data sets\footnote{\small \url{http://www.csie.ntu.edu.tw/~cjlin}}, are employed in this experiment. The reported experimental results are averaged over $10$ random repetitions of the data sets into training and test sets of the ratio $1:9$. We use the AG method to solve the S.TraceNorm formulation, and stop the iterative procedure of AG if the change of the objective values in two successive iterations is smaller than $10^{-8}$ or the iteration numbers larger than $10^5$. The regularization parameters $\alpha$ and $\beta$ are determined via double cross-validation from the set $\{10^{-2} \times i \}_{i=1}^{10} \cup \{10^{-1}
\times i \}_{i=2}^{10} \cup \{2 \times i\}_{i=1}^{10}$.
\begin{table*}[!t]
\caption{{\small Averaged performance (with standard derivation) comparison in terms of AUC, Macro F1, and Micro F1. Note that $n$, $d$, and $m$ denote the
sample size, dimensionality, and label number, respectively.}}
\centering
\begin{scriptsize}
\begin{tabular}{l|r|cccc}
\hline
\multicolumn{2}{c|}{Data Set}       & Business & Arts & Health & Scene   \\
\multicolumn{2}{c|}{(n, d, m)}      & {\scriptsize($9968, 16621, 17$)} & {\scriptsize($7441, 17973, 19$)} & {\scriptsize($9109, 18430, 14$)}  & {\scriptsize($2407, 294, 6$)} \\
\hline  \hline
          & S.TraceNorm & $85.42 \pm 0.31$  & $76.31 \pm 0.15$ & $86.18 \pm 0.56$ & $91.54 \pm 0.18$ \\
AUC       &    TraceNorm & $83.43 \pm 0.41$  & $75.90 \pm 0.27$ & $85.24 \pm 0.42$ & $90.33 \pm 0.24$ \\
          &      OneNorm & $81.95 \pm 0.26$  & $70.47 \pm 0.18$ & $83.60 \pm 0.32$ & $88.42 \pm 0.31$ \\
\hline
          & S.TraceNorm & $48.83 \pm 0.13$  & $32.83 \pm 0.25$ & $60.05 \pm 0.36$ & $51.65 \pm 0.33$ \\
Macro F$1$&    TraceNorm & $47.24 \pm 0.15$  & $31.90 \pm 0.31$ & $58.91 \pm 0.24$ & $50.59 \pm 0.08$ \\
          &      OneNorm & $46.28 \pm 0.25$  & $31.03 \pm 0.46$ & $58.01 \pm 0.18$ & $46.57 \pm 1.10$ \\
\hline
          & S.TraceNorm & $78.26 \pm 0.71$  & $42.91 \pm 0.27$ & $67.22 \pm 0.47$ & $52.83 \pm 0.35$ \\
Micro F$1$&    TraceNorm & $78.84 \pm 0.11$  & $42.08 \pm 0.11$ & $66.92 \pm 0.42$ & $52.06 \pm 0.49$ \\
          &      OneNorm & $78.16 \pm 0.17$  & $40.64 \pm 0.52$ & $66.37 \pm 0.19$ & $47.32 \pm 0.13$ \\
\hline
\end{tabular}
\end{scriptsize}
\label{tab:performance}
\vskip -0.1in
\end{table*}

We present the averaged performance of the
competing algorithms in Table~\ref{tab:performance}. The main observations are summarized as follows: (1) S.TraceNorm achieves the best performance on all benchmark data sets (except on Business data) in this experiment; this result demonstrates the effectiveness of the induced sparse low-rank structure for multi-label classification tasks; (2) TraceNorm outperforms OneNorm on all benchmark data sets; this result demonstrates the effectiveness of modeling a
shared low-rank structure for high-dimensional text and image data analysis.


\noindent {\bf Numerical Study} We study the practical convergence of AG and ADMM$2$ by solving Eq.~(\ref{eq:def-minimizer-to-SDP}) on Scene data. In our experiments, we observe that ADMM$1$ is much slower than ADMM$2$ and we thus only focus on ADMM$2$. Note that in AG, we set $\alpha = 1, \beta = 1$; in ADMM$2$, we set $\alpha = 1$, $\beta = 1$, $\rho_1 = \rho_2 = 10$. For other parameter settings, we observe similar trends.

In the first experiment, we compare AG and ADMM$2$ in term of the
practical convergence. We stop ADMM$2$ when the change of the
objective values in two successive iterations smaller than $10^{-4}$;
the attained objective value in ADMM$2$ is used as the stopping
criterion for AG, that is, we stop AG if the attained objective value
in AG is equal to or smaller than that objective value attained in
ADMM$2$. The convergence curves of ADMM$2$ and AG are presented in
the left plot of Figure~\ref{fig:convergence}. Clearly, we can
observe that AG converges much faster than ADMM$2$. In the second
experiment, we study the convergence of AG. We stop AG when the
change of the objective values in two successive iterations smaller
than $10^{-8}$. The convergence curves is presented in the middle
plot of Figure~\ref{fig:convergence}. We observe that AG converges
very fast, and its convergence speed is consistent with the
theoretical convergence analysis
in~\cite{Nesterov-IntrodConvexOpt-note98}.
\begin{figure*}[t]\centering
    \centering
    \includegraphics[width=1.85in,height=1in]{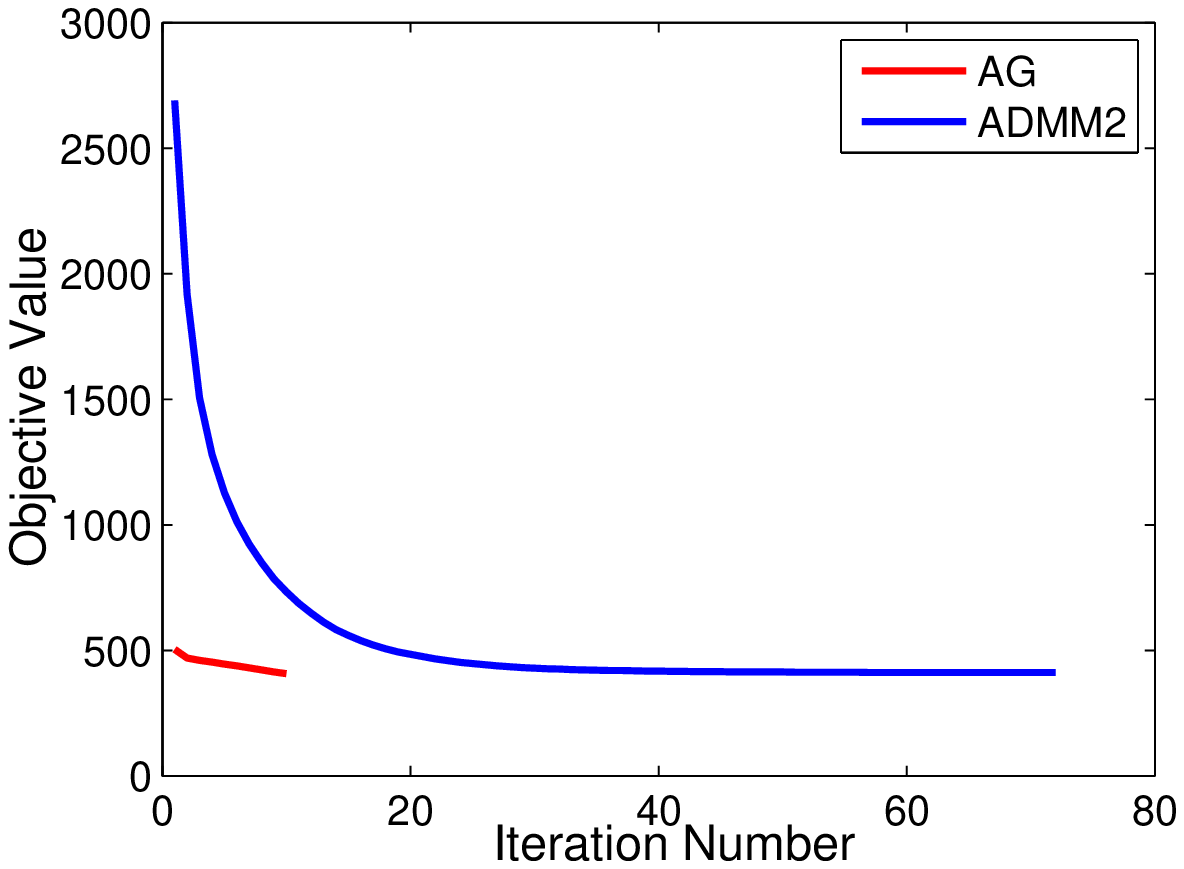}
    \hskip -0.1in
    \includegraphics[width=1.85in,height=1in]{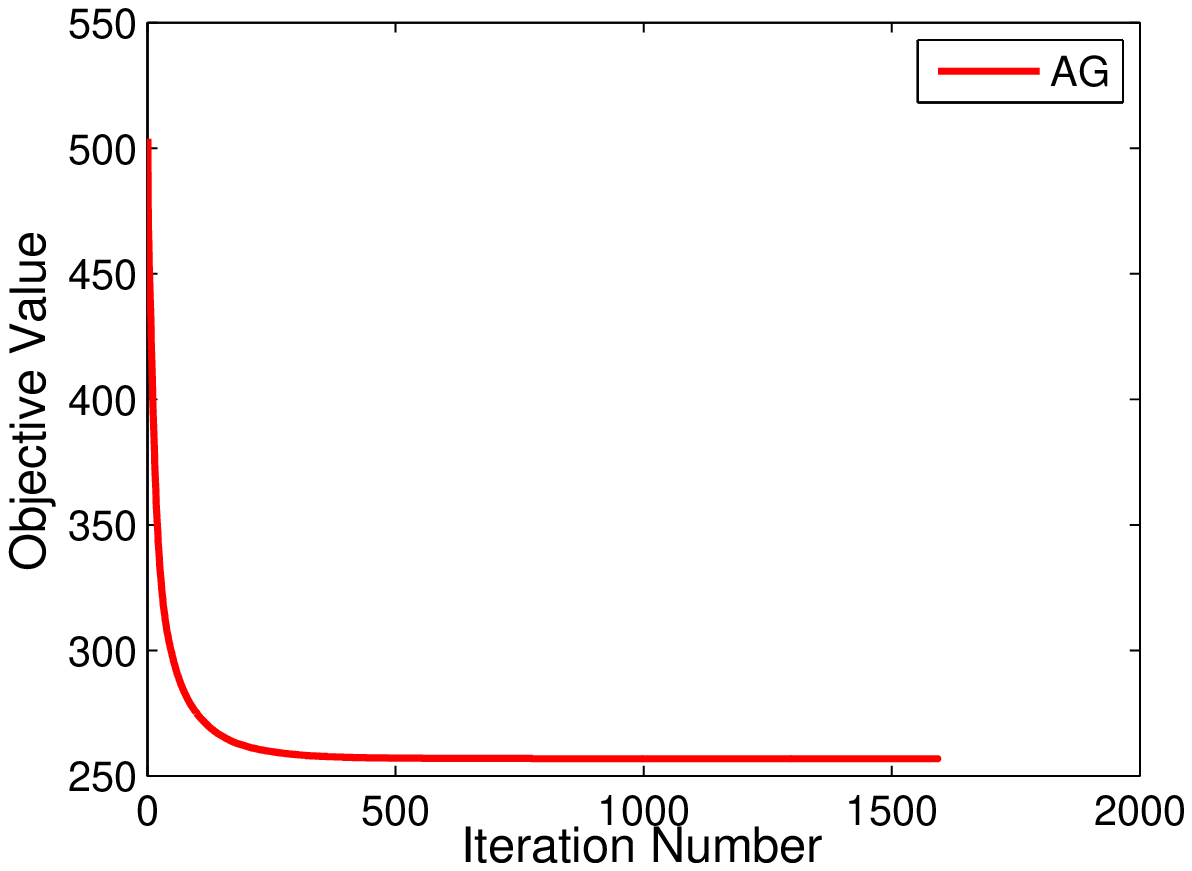}
    \hskip -0.1in
    \includegraphics[width=1.85in,height=1in]{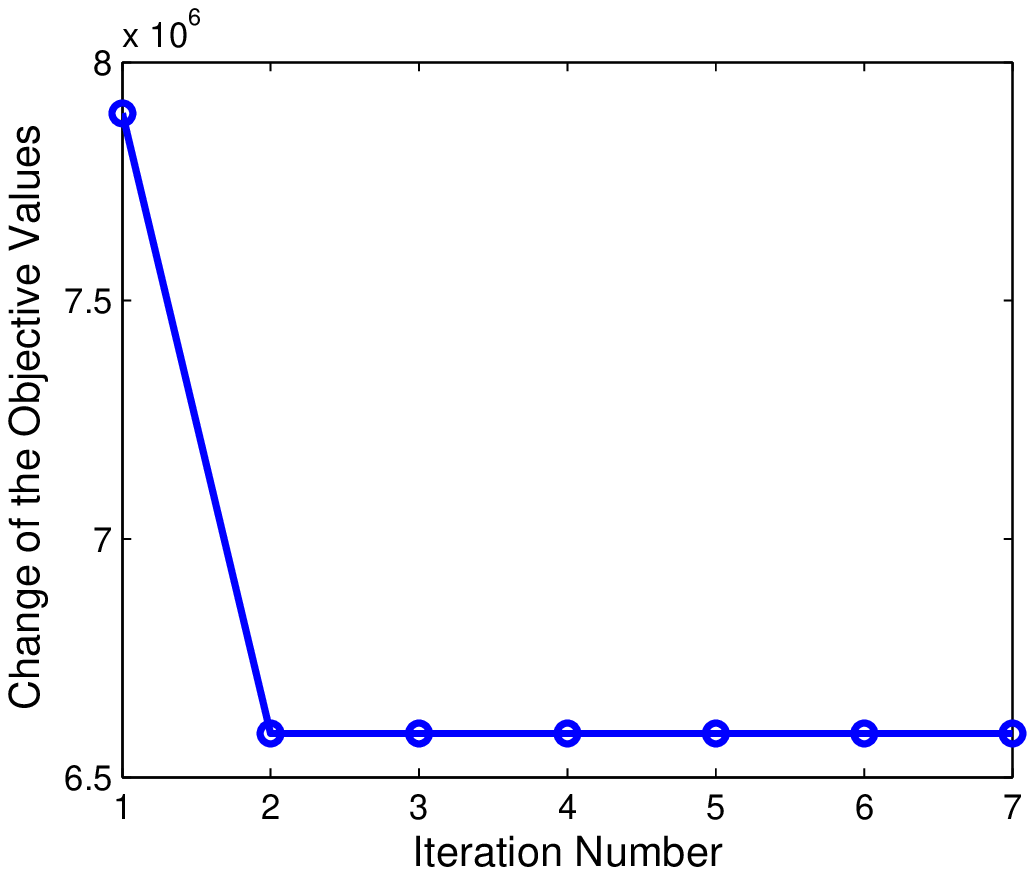}
    \caption{\small Convergence comparison of AG and ADMM$2$ for solving
Eq.~(\ref{eq:def-minimizer-to-SDP}) (left plot); convergence plot of
AG for solving Eq.~(\ref{eq:def-minimizer-to-SDP}) (middle plot); and
the alternating optimization algorithm for solving the dual
formulation of the proximal operator in Eq.~(\ref{eq:two-block})
(right plot).} \label{fig:convergence}
\vskip -0.1in
\end{figure*}


We also conduct numerical study on the alternating optimization
algorithm (in Section~\ref{subsec:alt-opt}) for solving the dual
formulation of the proximal operator in Eq.~(\ref{eq:two-block}).
Similarly, the alternating optimization algorithm is stopped when the
change of the objective values in two successive iterations smaller
than $10^{-8}$. For illustration, in Eq.~(\ref{eq:two-block}) we randomly generate the matrix $\widehat \Phi$ of size $10000$ by $5000$
from $\mathcal{N}(0,1)$; we then apply the alternating optimization
algorithm to solve Eq.~(\ref{eq:two-block}) and plot its convergence
curve in the right plot of Figure~\ref{fig:convergence}. Our
experimental results show that the alternating optimization algorithm
generally converges within $10$ iterations and our results
demonstrate the practical efficiency of this algorithm.


\section{Conclusion}

We study the problem of estimating multiple
predictive functions simultaneously in the nonparametric regression
setting. In our estimation scheme, each predictive function is
estimated using a linear combination of a dictionary of
pre-specified basis functions. By assuming that the coefficient
matrix admits a sparse low-rank structure, we formulate the function
estimation problem as a convex program with the trace norm and the $\ell_1$-norm
regularization. We propose to employ AG and ADMM algorithms to solve the function estimation problem and also develop efficient algorithms for the key components involved in AG and ADMM.  We derive a key property of the optimal solution to the convex program;
moreover, based on an assumption associated with the basis
functions, we establish a performance bound of the proposed function
estimation scheme using the composite regularization. Our
simulation studies demonstrate the effectiveness and the efficiency of the proposed formulation. In the future, we plan to derive a formal sparse oracle
inequality for the convex problem in
Eq.~(\ref{eq:def-minimizer-to-SDP}) as  in~\cite{bickel-simulLassoDantzig-AnnalsStat09}; we also plan to apply the proposed function estimation formulation to other real world applications.




{\lsp{0.88}
\small
\bibliography{SparseTraceNorm}
\bibliographystyle{plain}

\newpage
\setcounter{page}{1}

\begin{center}
\bf \Large Sparse Trace Norm Regularization: Supplemental Material
\end{center}

\section*{A. Operators $\mathcal{S}_0$ and $\mathcal{S}_1$}

We define two operators, namely $\mathcal{S}_0$ and $\mathcal{S}_1$,
on an arbitrary matrix pair (of the same size) based on Lemma $3.4$
in~\cite{Benjamin-trace-siam07}, as summarized in the following lemma.
\begin{appendLemma} \label{lem:sep-singular-value}
Given any $\Theta$ and $\Delta$ of size $h \times k$, let
$\mbox{rank} (\Theta) = r$ and denote the SVD of $\Theta$ as
\begin{equation*}
\Theta = U \left[
               \begin{array}{cc}
               \Sigma   & {\bf 0} \\
               {\bf 0}  & {\bf 0}
               \end{array}
         \right] V^T,
\end{equation*}
where $U \in \mathbb{R}^{h \times h}$ and $V \in \mathbb{R}^{k
\times k}$ are orthogonal, and $\Sigma \in \mathbb{R}^{r \times r}$
is diagonal consisting of the non-zero singular values on its main
diagonal. Let
\begin{equation*}
\widehat \Delta = U^T \Delta V = \left[
               \begin{array}{cc}
               {\widehat \Delta}_{11}   &   {\widehat \Delta}_{12} \\
               {\widehat \Delta}_{21}   &   {\widehat \Delta}_{22}
               \end{array}
         \right],
\end{equation*}
where ${\widehat \Delta}_{11} \in \mathbb{R}^{r \times r}$, ${\widehat \Delta}_{12} \in \mathbb{R}^{r \times (k-r)}$, ${\widehat \Delta}_{21}
\in \mathbb{R}^{(h-r) \times r}$, and ${\widehat \Delta}_{22}
\in \mathbb{R}^{(h-r) \times (k-r)}$. Define $\mathcal{S}_0$ and
$\mathcal{S}_1$ as
\begin{eqnarray*}
\mathcal{S}_0 (\Theta, \Delta)  =  U \left[
               \begin{array}{cc}
               {\widehat \Delta}_{11}   & {\widehat \Delta}_{12} \\
               {\widehat \Delta}_{21}  & {\bf 0}
               \end{array}
         \right] V^T, \,\,
\mathcal{S}_1 (\Theta, \Delta) = U \left[
               \begin{array}{cc}
               {\bf 0}   & {\bf 0} \\
               {\bf 0}   & {\widehat \Delta}_{22}
               \end{array}
         \right] V^T.
\end{eqnarray*}
Then the following conditions hold: $\mbox{rank} \left( \mathcal{S}_0 (\Theta, \Delta) \right) \le 2 r$, $\Theta \mathcal{S}_1 (\Theta, \Delta)^T = 0$, $\Theta^T \mathcal{S}_1 (\Theta, \Delta) = 0$.
\end{appendLemma}
The result presented in Lemma~\ref{lem:sep-singular-value} implies a condition
under which the trace norm on a matrix pair is additive. From
Lemma~\ref{lem:sep-singular-value} we can easily verify that
\begin{equation} \label{eq:separate-single-values}
\|\Theta + \mathcal{S}_1 ({\Theta, \Delta})\|_* =
\|\Theta\|_* + \|\mathcal{S}_1
(\Theta, \Delta)\|_*,
\end{equation}
for arbitrary $\Theta$ and $\Delta$ of the same size. To avoid clutter notation, we denote $\mathcal{S}_0 (\Theta, \Delta)$ by $\mathcal{S}_0 (\Delta)$, and $\mathcal{S}_1 (\Theta,
\Delta)$ by $\mathcal{S}_1 (\Delta)$ throughout this paper, as the appropriate
$\Theta$ can be easily determined from the context.

\section*{B. Bound on Trace Norm}

As a consequence of Lemma~\ref{lem:sep-singular-value}, we derive a bound on the trace norm of the matrices of interest as summarized below.
\begin{appendCorollary} \label{cor:sep-tracenorm}
Given an arbitrary matrix pair $\widehat \Theta$ and $\Theta$, let $\Delta = \widehat \Theta - \Theta$. Then
\begin{eqnarray*}
\|\widehat \Theta - \Theta \|_* +  \|\Theta\|_* - \|\widehat \Theta\|_* \le 2 \|\mathcal{S}_0 (\Delta)\|_*.
\end{eqnarray*}
\end{appendCorollary}

\begin{proof}
From Lemma~\ref{lem:sep-singular-value} we have $\Delta = \mathcal{S}_0 (\Delta) + \mathcal{S}_1 (\Delta)$ for the matrix pair $\Theta$ and $\Delta$. Moreover,
\begin{eqnarray} \label{eq:separate-theta-hat-1}
\|\widehat \Theta\|_* 
&  =   &  \|\Theta + \mathcal{S}_0 (\Delta) + \mathcal{S}_1 (\Delta) \|_*  \ge  \|\Theta + \mathcal{S}_1 (\Delta) \|_* - \| \mathcal{S}_0 (\Delta) \|_* \nonumber \\
&  =   &  \|\Theta \|_* + \|\mathcal{S}_1
(\Delta) \|_*  - \| \mathcal{S}_0 (\Delta)
\|_*,
\end{eqnarray}
where the inequality above follows from the triangle inequality and
the last equality above follows from
Eq.~(\ref{eq:separate-single-values}). Using the result in
Eq.~(\ref{eq:separate-theta-hat-1}), we have
\begin{eqnarray} \label{eq:TraceNorm-Ineq}
\|\widehat \Theta - \Theta \|_* +  \|\Theta\|_* - \|\widehat \Theta\|_* & \le &   \|\Delta\|_* +
\|\Theta\|_* -
\|\Theta \|_* - \|\mathcal{S}_1 (\Delta)
\|_*  + \| \mathcal{S}_0 (\Delta)
\|_*  \nonumber \\
& \le & 2 \|\mathcal{S}_0 (\Delta)\|_*. \nonumber
\end{eqnarray}
We complete the proof of this corollary.
\end{proof}

\section*{C. Bound on $\ell_1$-norm}

Analogous to the bound on the trace norm in Corollary~\ref{cor:sep-tracenorm}, we also derive a bound on the $\ell_1$-norm of the matrices of interest in the following lemma. For arbitrary matrices $\Theta$ and $\Delta$, we denote by $J (\Theta) =\{(i,j) \}$ the coordinate set (the location set of nonzero entries) of $\Theta$, and by $J(\Theta)_\bot$ the associated complement (the location set of zero entries); we denote by $\Delta_{J(\Theta)}$ the matrix of the same entries as $\Delta$ on the set $J(\Theta)$ and of zero entries on the set $J(\Theta)_\bot$. We now present a result associated with $J(\Theta)$ and $J(\Theta)_\bot$ in the following lemma. Note that a similar result for the vector case is presented in~\cite{bickel-simulLassoDantzig-AnnalsStat09}.
\begin{appendLemma} \label{lem:sep-nonzero-entry}
Given a matrix pair $\widehat \Theta$ and $\Theta$ of the same size, the inequality below always holds
\begin{equation} \label{eq:separate-one-norm}
\| \widehat \Theta - \Theta \|_1 + \| \Theta \|_1 - \| \widehat \Theta \|_1 \le 2 \| \widehat \Theta_{J(\Theta)} - \Theta_{J(\Theta)} \|_1.
\end{equation}
\end{appendLemma}

\begin{proof}
It can be verified that the inequality
\begin{equation*}
\| \Theta_{J(\Theta)} \|_1 - \| \widehat \Theta_{J(\Theta)} \|_1 \le \| (\widehat \Theta - \Theta)_{J(\Theta)} \|_1
\end{equation*}
and the equalities
\begin{equation*}
\Theta_{J(\Theta)_{\bot}} = {\bf 0}, \,\,\, \| (\widehat \Theta - \Theta)_{J(\Theta)_\bot} \|_1 - \| \widehat \Theta_{J(\Theta)} \|_1 = {\bf 0}
\end{equation*}
hold. Therefore we can derive
\begin{eqnarray*}
&   & \| \widehat \Theta - \Theta \|_1 + \| \Theta \|_1 - \| \widehat \Theta \|_1  \\
& = & \| (\widehat \Theta - \Theta)_{J(\Theta)} \|_1 + \| (\widehat \Theta - \Theta)_{J(\Theta)_\bot} \|_1  +  \| \Theta_{J(\Theta)} \|_1 + \| \Theta_{J(\Theta)_\bot} \|_1 - \| \widehat \Theta_{J(\Theta)} \|_1  - \| \widehat \Theta_{J(\Theta)_\bot} \|_1  \\
& \le & 2 \| (\widehat \Theta - \Theta)_{J(\Theta)} \|_1.
\end{eqnarray*}
This completes the proof of this lemma.
\end{proof}

\section*{D. Concentration Inequality}

\begin{appendLemma} \label{lem:event-A-hold}
Let $\sigma_{\scriptsize{X(l)}}$ be the maximum singular value of
the matrix $\mathcal{G}_X \in \mathbb{R}^{n \times h}$; let $W \in \mathbb{R}^{n \times k}$ be the matrix of
i.i.d entries as $w_{ij} \sim \mathcal{N} (0, \sigma_w^2 )$. Let
$
\lambda = {2 \sigma_{\scriptsize{X(l)}} \sigma_w \sqrt{n}} \left( 1 +
\sqrt{{k}/{n}} + t \right) / {N}.
$
Then
\begin{equation*}
\Pr \left( \|W^T \mathcal{G}_X \|_2 / N \le
{\lambda}/{2} \right) \ge 1 - \exp\left(- n t^2 / 2
\right).
\end{equation*}
\end{appendLemma}
\begin{proof}
It is known~\cite{Szarek-random-matrix-journal-compleixity90} that a Gaussian matrix $\widehat W \in \mathbb{R}^{n \times k}$ with $n
\ge k$ and $\hat w_{ij} \sim \mathcal{N} (0, {1}/{n})$ satisfies
\begin{equation} \label{eq:concen-ineq}
\mbox{Pr} \left( \|\widehat W\|_2 > 1 + \sqrt{{k}/{n}} + t \right)
\le \exp \left(- n t^2 / 2 \right),
\end{equation}
where $t$ is a universal constant. From the definition of the
largest singular value, there exist a vector $b \in \mathbb{R}^h$ of length $1$, i.e., $\|b\|_2 = 1$, such that $\|W^T \mathcal{G}_X\|_2 = \|W^T \mathcal{G}_X b\|_2 \le \| W \|_2
\|\mathcal{G}_X b\|_2 \le \sigma_{X(l)}\|W\|_2$. Since $w_{ij} / \left( \sigma_w \sqrt{n} \right) \sim \mathcal{N} (0, 1 / n)$, we have
\begin{eqnarray*}
\mbox{Pr} \left( \left\|W^T \mathcal{G}_X \right\|_2 / N
> {\lambda}/{2} \right)
 \le  \mbox{Pr} \left( {\sigma_{\scriptsize{X(l)}}}
\left\| W \right\|_2 / {N } > {\lambda}/{2} \right).
\end{eqnarray*}
Applying the result in Eq.~(\ref{eq:concen-ineq}) into the inequality above, we complete the proof of this lemma.
\end{proof}

\section*{E. Implementations of the Alternating Direction Method of Multipliers for Solving Eq.~(\ref{eq:def-minimizer-to-SDP})}

We employ two variants of the Alternating Direction Method of Multipliers~(ADMM) to solve the Eq.~(\ref{eq:def-minimizer-to-SDP}). The key difference lies in the use of different numbers of auxiliary variables to separate the smooth components from the non-smooth components of the objective function in Eq.~(\ref{eq:def-minimizer-to-SDP}).

\subsection*{E.1 The First Implementation: ADMM$1$}
By adding an auxiliary variable $\Psi$, we reformulate Eq.~(\ref{eq:def-minimizer-to-SDP}) as
\begin{eqnarray} \label{eq:reform-SpTrNorm}
\min_{\Theta,\Psi}  &&  \widehat S(\Theta) +  \alpha \|\Psi\|_* + \beta \|\Theta\|_1 \nonumber \\
\mbox{subject to}   &&  \Theta = \Psi.
\end{eqnarray}
The augmented Lagrangian of Eq.~(\ref{eq:reform-SpTrNorm}) can be expressed as
\begin{equation} \label{eq:reform-SpTrNorm-Lag}
\mathcal{L}_{\rho}^1 (\Theta, \Psi, \Gamma) = \widehat S(\Theta) +  \alpha \|\Psi\|_* + \beta \|\Theta\|_1 + \langle \Theta - \Psi, \Gamma \rangle + \frac{\rho}{2} \| \Theta - \Psi \|_F^2.
\end{equation}
To solve Eq.~(\ref{eq:reform-SpTrNorm}), ADMM$1$ consists of the following iterations:
\begin{eqnarray}
\Theta_{k+1} & = & \arg \min_\Theta \mathcal{L}_{\rho}^1 (\Theta, \Psi_k, \Gamma_k),  \label{eq:alg1-update-Theta}\\
\Psi_{k+1}   & = & \arg \min_{\Psi} \mathcal{L}_\rho^1 (\Theta_{k+1}, \Psi, \Gamma_k), \label{eq:alg1-update-Psi} \\
\Gamma_{k+1} & = & \Gamma_k + \rho \left( \Theta_{k+1} - \Psi_{k+1} \right), \label{eq:alg1-update-Gamma}
\end{eqnarray}
where $\Theta_k$, $\Psi_k$, and $\Gamma_k$ denote the intermediate solutions of ADMM$1$ at the $k$-th iteration, and $\rho$ is a pre-specified constant.

Specifically, if we employ the least squares loss, i.e., $\widehat
S(\Theta) = \| \mathcal{G}_X \Theta - Y \|_F^2 / N$, the
optimization problems in Eqs.~(\ref{eq:alg1-update-Theta})
and~(\ref{eq:alg1-update-Gamma}) can be efficiently solved as below.

{\bf Update on $\Theta$} The optimal $\Theta_{k+1}$ to Eq.~(\ref{eq:alg1-update-Theta}) can be obtained via
\begin{eqnarray}
\Theta_{k+1} = \arg \min_\Theta  \left( \frac{1}{N} \| \mathcal{G}_X \Theta - Y \|_F^2 + \beta \|\Theta\|_1 + \langle \Theta, \Gamma_k \rangle + \frac{\rho}{2} \| \Theta - \Psi_k \|_F^2 \right),
\end{eqnarray}
which can be efficiently solved via the gradient-type methods~\cite{Beck-fast-09,Nesterov:2007}.

{\bf Update on $\Psi$} The optimal $\Psi_{k+1}$ to Eq.~(\ref{eq:alg1-update-Psi}) can be obtained via
\begin{eqnarray}
\Psi_{k+1} = \arg \min_{\Psi} \left( \alpha \|\Psi\|_* - \langle \Psi, \Gamma_k \rangle + \frac{\rho}{2} \| \Theta_{k+1} - \Psi \|_F^2 \right). \nonumber
\end{eqnarray}
The optimization problem above admits an analytical solution~\cite{Benjamin-trace-siam07}. Assume $\mbox{rank} \left( \Theta_{k+1} + \Gamma_k / {\rho} \right) = r$. Let $\Theta_{k+1} + \Gamma_k / {\rho} = U_r \Sigma_r V_r^T$ be the singular value decomposition of $\Theta_{k+1} + \Gamma_k / {\rho}$, where $U_r$ and $V_r$ consist of respectively $r$ orthonormal columns, and $\Sigma_r = \mbox{diag} \left\{ (\sigma_1, \sigma_2, \cdots, \sigma_r) \right\}$. Then the optimal $\Psi_{k+1}$ is given by
\begin{equation}
\Psi_{k+1} = U_r \hat \Sigma V_r^T, \,\, \hat \Sigma = \mbox{diag} \left\{ \left( \sigma_i - \frac{\alpha}{\rho} \right)_+ \right\},
\end{equation}
where $(x)_+ = x$ if $x > 0$ and $(x)_+ = 0$ otherwise.

\subsection*{E.2 The Second Implementation: ADMM$2$} \label{sec:ADMM2}

By adding two auxiliary variables $\Psi^1$ and $\Psi^2$, we reformulate Eq.~(\ref{eq:def-minimizer-to-SDP}) as
\begin{eqnarray} \label{eq:reform-SpTrNorm-2}
\min_{\Theta,\Psi^1, \Psi^2}  &&  \widehat S(\Theta) +  \alpha \|\Psi^1\|_* + \beta \|\Psi^2\|_1 \nonumber \\
\mbox{subject to}   &&  \Theta = \Psi^1, \,\, \Theta = \Psi^2.
\end{eqnarray}
Similarly, the augmented Lagrangian of Eq.~(\ref{eq:reform-SpTrNorm-2}) can be expressed as
\begin{eqnarray*} 
&     &  \mathcal{L}_{\rho_1, \rho_2}^2 (\Theta, \Psi^1, \Psi^2, \Gamma^1, \Gamma^2) \\
& =   &  \widehat S(\Theta) +  \alpha \|\Psi^1\|_* + \beta \|\Psi^2\|_1 + \langle \Theta - \Psi^1, \Gamma^1 \rangle + \langle \Theta - \Psi^2, \Gamma_2 \rangle + \frac{\rho_1}{2} \| \Theta - \Psi^1 \|_F^2 + \frac{\rho_2}{2} \| \Theta - \Psi^2 \|_F^2.
\end{eqnarray*}
To solve Eq.~(\ref{eq:reform-SpTrNorm-2}), ADMM$2$ consists of the following iterations:
\begin{eqnarray}
\Theta_{k+1} & = & \arg \min_\Theta \mathcal{L}_{\rho_1, \rho_2}^2 (\Theta, \Psi_k^1, \Psi_k^2, \Gamma_k^1, \Gamma_k^2),  \label{eq:alg2-update-Theta}\\
\left( \Psi_{k+1}^1, \Psi_{k+1}^2 \right)   & = & \arg \min_{\Psi^1, \Psi^2} \mathcal{L}_{\rho_1, \rho_2}^2 (\Theta_{k+1}, \Psi^1, \Psi^2, \Gamma_k^1, \Gamma_k^2), \label{eq:alg2-update-Psi} \\
\Gamma_{k+1}^1 & = & \Gamma_k^1 + \rho_1 \left( \Theta_{k+1} - \Psi_{k+1}^1 \right), \label{eq:alg2-update-Gamma-1} \\
\Gamma_{k+1}^2 & = & \Gamma_k^2 + \rho_2 \left( \Theta_{k+1} - \Psi_{k+1}^2 \right), \label{eq:alg2-update-Gamma-2}
\end{eqnarray}
where $\Theta_k$, $\Psi_k^1$, $\Psi_k^2$, $\Gamma_k^1$, and $\Gamma_k^2$ denote the intermediate solutions at the $k$-th iteration of the ADMM$2$ method.

Specifically, if we employ $\widehat S(\Theta) = \| \mathcal{G}_X \Theta - Y \|_F^2 / N$ as the loss function in Eq.~(\ref{eq:reform-SpTrNorm-2}), the optimization problems in Eqs.~(\ref{eq:alg2-update-Theta}),~(\ref{eq:alg2-update-Psi}),~(\ref{eq:alg2-update-Gamma-1}), and~(\ref{eq:alg2-update-Gamma-2}) can be efficiently solved as below.

{\bf Update on $\Theta$} The optimal $\Theta_{k+1}$ to Eq.~(\ref{eq:alg2-update-Theta}) can be obtained via
\begin{eqnarray*}
\Theta_{k+1} & = & \arg \min_\Theta \left( \frac{1}{N} \| \mathcal{G}_X \Theta - Y \|_F^2 + \langle \Theta, \Gamma_k^1 + \Gamma_k^2 \rangle + \frac{\rho_1}{2} \| \Theta - \Psi_k^1 \|_F^2 + \frac{\rho_2}{2} \| \Theta - \Psi_k^2 \|_F^2 \right). 
\end{eqnarray*}
Note that the optimal $\Theta_{k+1}$ can be obtained via solving a systems of linear equations.

{\bf Update on $\Psi^1$ and $\Psi^2$} The optimal $\Psi_{k+1}^1$ and $\Psi_{k+1}^1$ to Eq.~(\ref{eq:alg2-update-Psi}) can be obtained via
\begin{eqnarray}
\Psi_{k+1}^1 & = & \arg \min_{\Psi^1} \left(\alpha \|\Psi^1\|_* - \langle \Psi^1, \Gamma_k^1 \rangle + \frac{\rho_1}{2} \| \Theta_{k+1} - \Psi^1 \|_F^2 \right), \label{eq:compute-Gamma-1}\\
\Psi_{k+1}^2 & = & \arg \min_{\Psi^2} \left( \beta \|\Psi^2\|_1 - \langle \Psi^2, \Gamma_k^2 \rangle + \frac{\rho_2}{2} \| \Theta_{k+1} - \Psi^2 \|_F^2 \right). \label{eq:compute-Gamma-2}
\end{eqnarray}
It can be verified that Eq.~(\ref{eq:compute-Gamma-1}) admits an analytical solution. Assume $\mbox{rank} \left( \Theta_{k+1} + \Gamma_k^1 / {\rho_1} \right) = r$. Let $\Theta_{k+1} + \Gamma_k^1 / {\rho_1} = U_r \Sigma_r V_r^T$ be the singular value decomposition of $\Theta_{k+1} + \Gamma_k^1 / {\rho_1}$, where $U_r$ and $V_r$ consist of respectively $r$ orthonormal columns, and $\Sigma_r = \mbox{diag} \left\{ (\sigma_1, \sigma_2, \cdots, \sigma_r) \right\}$. Then the optimal $\Psi_{k+1}^1$ is given by
\begin{equation}
\Psi_{k+1}^1 = U_r \hat \Sigma V_r^T, \,\, \hat \Sigma = \mbox{diag} \left\{ \left( \sigma_i - \frac{\alpha}{\rho_1} \right)_+ \right\},
\end{equation}
where $(x)_+ = x$ if $x > 0$ and $(x)_+ = 0$ otherwise.

Moreover, it can also be verified that Eq.~(\ref{eq:compute-Gamma-2}) admits an analytical solution. Let $\psi$, $\theta$, and $\gamma$ be the entries of $\Psi_{k+1}^2$, $\Theta_{k+1}$, and $\Gamma_k^2$ at the same coordinates. The optimal $\psi$ is given by
\begin{eqnarray}
\psi = \left \{ \begin{array} {cc}
                 \theta + \frac{1}{\rho_2} \left( \gamma - \beta \right)   &  \theta + \frac{1}{\rho_2} \gamma > \frac{1}{\rho_2} \beta \\
                 0 & - \frac{1}{\rho_2} \beta \le \theta + \frac{1}{\rho_2} \gamma \le \frac{1}{\rho_2} \beta \\
                 \theta + \frac{1}{\rho_2} \left( \gamma + \beta \right)  & \theta + \frac{1}{\rho_2} \gamma < - \frac{1}{\rho_2} \beta
                 \end{array}. \right.
\end{eqnarray}

}

\end{document}